\documentclass[12pt]{alt2024}

\usepackage{amsmath}
\usepackage{amssymb}
\usepackage{fancybox}
\usepackage{eucal}
\usepackage{amscd}
\usepackage{hyperref}
\usepackage{wasysym}
\usepackage{url}
\usepackage{enumitem}
\usepackage{amssymb,}
\usepackage[all,cmtip]{xy}
\usepackage{xcolor}
\usepackage{soul}
\input amssym.def \input amssym

\newtheorem {thm}{Theorem}[section]

\newtheorem {cl}[thm]{Claim}

\newtheorem{fact}[thm]{Fact}
\newtheorem {lem}[thm]{Lemma}

\newtheorem{defn}[thm]{Definition}

\def\Ind{\setbox0=\hbox{$x$}\kern\wd0\hbox to 0pt{\hss$\mid$\hss} \lower.9\ht0\hbox to 0pt{\hss$\smile$\hss}\kern\wd0}
\def\Notind{\setbox0=\hbox{$x$}\kern\wd0\hbox to 0pt{\mathchardef \nn=12854\hss$\nn$\kern1.4\wd0\hss}\hbox to 0pt{\hss$\mid$\hss}\lower.9\ht0 \hbox to 0pt{\hss$\smile$\hss}\kern\wd0}

\numberwithin{equation}{section}

\newcommand{\m}{\mathbb }
\newcommand{\mc}{\mathcal }

\DeclareMathOperator{\Ldim}{Ldim}

\DeclareMathOperator{\dom}{dom}

\title{Applications of Littlestone dimension to query learning and to compression}
\usepackage{times}

\altauthor{%
 \Name{Hunter Chase} \Email{hsachase@gmail.com}\\
 \Name{James Frietag} \Email{jfreitag@uic.edu}\\
 \Name{Lev Reyzin} \Email{lreyzin@uic.edu}\\
 \addr Department of Mathematics, Statistics, and Computer Science \\ University of Illinois at Chicago
}

\begin{document}

\maketitle

\begin{abstract}
In this paper we give several applications of Littlestone dimension. The first is to the model of \cite{angluin2017power}, where we extend their results
for learning by equivalence queries with random counterexamples. Second, we extend that model to infinite concept classes with an additional source of randomness. Third, we give improved results
on the relationship of Littlestone dimension to classes with extended $d$-compression schemes, proving a strong version of a conjecture of \cite{floyd1995sample} for Littlestone dimension.

\end{abstract}

\section{Introduction}

In query learning, a learner attempts to identify an unknown concept from a collection via a series of data requests called queries. Typically, algorithms designed for learning in this setting attempt to bound the number of required queries to identify the target concept in the worst case scenario. If one imagines the queries of the learner being answered by a teacher, the usual setup imagines the teacher answering queries in an adversarial manner, with minimally informative answers. Alternatively, for a given algorithm, the bounds for the traditional model are on the \emph{worst-case answers} over \emph{all potential targets}. In variations of the model, one of these two factors is usually modified. 

For instance, \cite{kumar2021teaching} studies the case in which the answers are assumed to be maximally informative in a certain sense. In this manuscript, we first work in the setup originating with \cite{angluin2017power}, where we assume that the answers to the queries are randomly selected with respect to some fixed probability distribution.

Consider a concept class $\mc C = \{C_1, \ldots , C_n \}, $ subsets of a fixed set $X$. Fix a target concept $A \in \mc C$. An \emph{equivalence query} consists of the learner submitting a hypothesis $B \in \mathcal{C}$ to a teacher, who either returns \emph{yes} if $A = B$, or a counterexample $x \in A \triangle B$. In the former case, the learner has learned $A$, and in the latter case, the learner uses the new information to update and submit a new hypothesis. 

\cite{angluin2017power} fix a probability distribution $\mu$ on $X$ and assume that the teacher selects the counterexamples randomly with respect to $\mu$ restricted to $A \triangle B$. 
They show that for a concept class $\mc C$ of size $n$, there is an algorithm in which the expected number of queries to learn any concept is at most $\log _2 (n).$ It is natural to wonder whether there is a combinatorial notion of dimension which can be used to bound the expected number of queries independent of the size of the class - perhaps even in infinite classes. In fact, \cite{angluin2017power} (Theorem 25) already consider this, and show that the VC-dimension of the concept class is a lower bound on the number of expected queries. On the other hand, \cite{angluin2017power} (Theorem 26), using an example of \cite{littlestone1988learning}, show that the VC-dimension \emph{cannot} provide an upper bound for the number of queries. 

The motivation for bounds depending on some notion of dimension rather than the number of concepts is two-fold: 
\begin{itemize} 
\item Many combinatorial notions of dimension (e.g. Littlestone or VC) of a class $\mc C$ can be small while $|\mc C|$ is large.
\item Investigating this model of learning in settings where $\mc C$ is an infinite class will require methods and bounds which do not use $|\mc C|$. 
\end{itemize} 

We show that the Littlestone dimension provides such an upper bound; we give an algorithm which yields a bound which is linear in the Littlestone dimension for the expected number of queries needed to learn any concept. In Section \ref{randomeq} we establish the bounds for finite concept classes $\mc C$. 

In Section \ref{learningmore} we give a specific example which shows finite Littlestone dimension of a infinite class $\mc C$ is not sufficient to guarantee learnability of the class in the model of \cite{angluin2017power}. That is, we show the expected number of queries is impossible to bound over all target concepts even in very simple infinite classes. Suppose that the target concept is itself selected randomly with respect to some (perhaps unrelated to the feedback mechanism) probability distribution. In this case, we give an algorithm so that the expected number of queries (over both sources of randomness) is at most $\tilde{O}(d)$ where $d $ is the Littlestone dimension of the class $\mc C$. This result uses the bounds developed in Section \ref{randomeq} in an essential way, in particular by using the finite class's Littlestone dimension instead of its size.

In Section \ref{compression}, we give another application of Littlestone dimension - to compression schemes which answers a question of \cite{johnson2010compression} on $d$-compression with $b$ extra bits, a notion originating with \cite{floyd1995sample}. The existence of a $d$-compression is closely related to various notions of learning; $d$-compressibility of a class $\mc C$ implies the class has VC-dimension at most $d$. A famous conjecture of \cite{floyd1995sample} asks if the every VC-class has a $d$-compression where $d$ is the VC-dimension.\footnote{Resolving whether there is an $O(d)$ compression has a reward of 600 dollars \cite{10.1007/978-3-540-45167-9_60}.} Our result in Section \ref{compression} proves a strong version of the conjecture for Littlestone dimension.


\section{Random counterexamples and EQ-learning} \label{randomeq}

In this section, we essentially work in the setting of \cite{angluin2017power} with slightly different notation. Throughout this section, let $X$ be a finite set, let $\mc C$ be a set system on $X$, and let $\mu$ be a probability measure on $X$. For $A,B \in \mc C$, let $$\Delta (A,B) = \{x \in X \, | \, A(x) \neq B(x) \}$$ denote the symmetric difference of $A$ and $B$. 

\begin{defn} We denote, by $\mc C_{\bar x = \bar i}$ for $\bar x \in X^n$ and $\bar i \in \{0,1\}^n$, the set system  
$ \{A \in \mc C \, | \, A(x_j )= i_j , \, j=1, \ldots , n \}.$ 
For $A \in \mc C$ and $a \in X$, we let $$u(A,a) = \Ldim (\mc C) - \Ldim (\mc C _{a = A(a)} ).$$ 
\end{defn} 
For any $a \in X,$ either $\mc C_{a=1}$ or $\mc C_{a=0}$ has Littlestone dimension strictly less than that of $\mc C$ and so: 
\begin{lem} \label{ulemma} For $A, B \in \mc C$ and $a \in X$ with $A(a) \neq B(a),$ 
$$u(A,a ) + u (B, a) \geq 1.$$ 
\end{lem} 
Next, we define a directed graph which is similar to the \emph{elimination graph} of \cite{angluin2017power}. 

\begin{defn}
We define the \emph{thicket query graph} $G_{TQ}(\mc C , \mu )$ to be the weighted directed graph on vertex set $\mc C$ such that the directed edge from $A$ to $B$ has weight $d(A,B)$ equal to the expected value of $\Ldim(\mc C)- \Ldim(\mc C _{x = B(x)}) $ over $x \in \Delta (A,B)$ with respect to the distribution $\mu |_{\Delta(A,B)}.$ \footnote{Here one should think of the query by the learner as being $A$, and the actual hypothesis being $B$. The teacher samples from $\Delta(A,B)$, and the learner now knows the value of the hypothesis on $x$.} 
\end{defn} 

\begin{defn} The \emph{query rank} of $A \in \mc C$ is defined as: 
$\inf _{B \in \mc C } (d(A,B)).$ 
\end{defn}

\begin{lem} \label{Lem13} For any $A \neq B \in \mc C$, $d(A, B)+d(B,A) \geq 1.$
\end{lem} 
\begin{proof} 
Noting that $\Delta (A,B) = \Delta (B,A),$ and using Lemma \ref{ulemma}:
\begin{eqnarray*} d(A,B) + d(B,A) & = & \sum _{a \in \Delta (A,B)} \frac{\mu(a)}{\mu (\Delta (A,B))} (u(A,a)+u(B,a)) \\
& \geq  & \sum _{a \in \Delta (A,B)} \frac{\mu(a)}{\mu (\Delta (A,B))}\\ & = & 1. \end{eqnarray*} 
\end{proof}


\begin{defn}[\cite{angluin2017power}, Definition 14] 
Let $G$ be a weighted directed graph and $ l \in \m N, \,  l >1.$ A \emph{deficient $l$-cycle} in $G$ is a sequence $v_0, \ldots v_{l-1}$ of distinct vertices such that for all $i \in [l]$, $d(v_i , v_{(i+1) \, (\mod l) } ) \leq \frac{1}{2}$ with strict inequality for at least one $i \in [l]$. 
\end{defn} 


The next result is similar to Theorems 16 (the case $l=3$) and Theorem 17 (the case $l >3$) of \cite{angluin2017power}, but our proof is rather different (note that the case $l=2$ follows easily from Lemma \ref{Lem13}).

\begin{thm} \label{nocycles} The thicket query graph $G_{TQ} (\mc C , \mu)$ has no degenerate $l$-cycles for $l \geq 2.$ 
\end{thm} 

The analogue of Theorem 16 can be adapted in a very similar manner to the technique employed by \cite{angluin2017power}. However, the analogue of the proof of Theorem 17 falls apart in our context; the reason is that Lemma \ref{ulemma} is analogous to Lemma 6 of \cite{angluin2017power} (and Lemma \ref{Lem13} is analogous to Lemma 13 of \cite{angluin2017power}), but our lemmas involve inequalities instead of equations. The inductive technique of \cite[Theorem 17]{angluin2017power} is to shorten degenerate cycles by considering the weights of a particular edge in the elimination graph along with the weight of the edge in the opposite direction. Since one of those weights being large forces the other to be small (by the \emph{equalities} of their lemmas), the induction naturally separates into two useful cases. In our thicket query graph, things are much less tightly constrained - one weight of an edge being large does not force the weight of the edge in the opposite direction to be small. However, the technique employed in our proof seems to be flexible enough to adapt to prove Theorems 16 and 17 of \cite{angluin2017power}.

\begin{proof} Suppose the vertices in the degenerate $l$-cycle are $A_0, \ldots , A_{l-1} $. 
By the definition of degenerate cycles and $d(-,-),$ we have, for each $i \in \m Z / l \m Z$, that $$\sum _{a \in \Delta (A_i, A_{i+1}) } \frac{\mu(a)}{\mu (\Delta (A_i ,A_{i+1}))} u(A_i,a) \leq \frac{1}{2}.$$ Clearing the denominator we have

\begin{equation} \label{firsteqn} \sum _{a \in \Delta (A_i,A_{i+1})} \mu(a) u(A_i,a) \leq \frac{1}{2} \mu (\Delta (A_i,A_{i+1})). \end{equation} 
\emph{Note that throughout this argument, the coefficients are being calculated modulo $l$.} Notice that for at least one value of $i$, the inequality in \ref{firsteqn} must be strict. 

Let $G, H$ be a partition of $$\mc X =  \{ A_1, \ldots , A_l \}.$$ Now define $$D (G, H) := \left \{a \in X \, | \, \forall A_1,B_1 \in G, \, \forall A_2, B_2 \in H,\, A_1(a) = B_1 (a), \, A_2(a) = B_2 (a),  A_1(a) \neq A_2 (a) \right \}.$$  

The following fact follows from the definition of $\Delta (A, B) $ and $D(-,-)$. 
\begin{fact} \label{usefulfact}  The set $\Delta (A_i , A_{i+1} )$ is the disjoint union, over all partitions of $\mc X$ into two pieces $G,H$ such that $A_i \in G$ and $A_{i+1} \in H$ of the sets $D(G,H).$ 
\end{fact} 

Now, take the sum of the inequalities \ref{firsteqn} as $i$ ranges from $1$ to $l$. On the LHS of the resulting sum, we obtain $$\sum _{i=1} ^ l \left(  \sum _{G, H \text{ a partition of $\mc X$},\, A_i \in G, A_{i+1} \in H} \left( \sum_{a \in D(G,H)}  \mu(a) u(A_i,a) \right) \right).$$
On the RHS of the resulting sum we obtain 
$$ \frac{1}{2} \sum _{i=1} ^ l \left(  \sum _{G, H \text{ a partition of $\mc X$},\, A_i \in G, A_{i+1} \in H} \left( \sum_{a \in D(G,H)}  \mu(a) \right) \right).$$
Given a partition $G,H$ of $\{ A_1, \ldots , A_l \}$ we note that the term $D(G,H) = D(H,G)$ appears exactly once as an element of the above sum for a fixed value of $i$ exactly when $A_i \in G$ and $A_{i+1} \in H$ or $A_i \in H$ and $A_{i+1} \in G.$ 

Consider the partition $G,H$ of $\mc X$. Suppose that $A_j , A_{j+1} , \ldots , A_k $ is a block of elements each contained in $G$, and that $A_{j-1}, A_{k+1} $ are in $H$. Now consider the terms 
$i=j-1$ and $i=k$ of the above sums (each of which where $D(G,H)$ appears). 

On the left hand side, we have $\sum _{a \in D(G,H) } \mu (a) u ( A_{j-1} ,a )) $ and $\sum _{a \in D(G,H) } \mu (a) u ( A_{k} ,a )) $. Note that for $a \in D(G,H)$, we have $a \in \Delta (A_{j-1}, A_k).$ So, by Lemma \ref{ulemma}, we have

$$\sum _{a \in D(G,H) } \mu (a) u ( A_{j-1} ,a ) + \sum _{a \in D(G,H) } \mu (a) u ( A_{k} ,a )  \geq \sum _{a \in D(G,H) } \mu (a) .$$
On the RHS, we have $$\frac{1}{2} \left(\sum _{a \in D(G,H) } \mu (a) +\sum _{a \in D(G,H) } \mu (a)  \right) = \sum _{a \in D(G,H) } \mu (a) .$$ For each $G,H$ a partition of $X$, the terms appearing in the above sum occur in pairs as above by Fact \ref{usefulfact}, and so, we have the the LHS is at least as large as the RHS of the sum of inequalities \ref{firsteqn}, which is impossible, since one of the inequalities must have been strict by our degenerate cycle. 
\end{proof}

\begin{thm} \label{qrlowbound} There is at least one element $A \in \mc C$ with query rank at least $\frac{1}{2}$. 
\end{thm} 
\begin{proof} If not, then for every element $A \in \mc C$, there is some element $B \in \mc C$ such that $d(A,B) < \frac{1}{2}$. So, pick, for each $A \in \mc C$, an element $f(A)$ such that $d(A,f(A)) < \frac{1}{2}.$ Now, fix $A \in \mc C$ and consider the sequence of elements of $\mc C$ given by $(f^i (A))$; since $\mc C$ is finite, at some point the sequence repeats itself. So, take a list of elements $B, f(B) , \ldots , f^n (B) = B$. By construction, this yields a bad cycle, contradicting Theorem \ref{nocycles}. 
\end{proof}

\subsection{The thicket max-min algorithm}\label{app:Thicketminmax}
In this subsection we show how to use the lower bound on query rank proved in Theorem \ref{qrlowbound} to give an algorithm which yields the correct concept in linearly (in the Littlestone dimension) many queries from $\mathcal{C}$. The approach is fairly straightforward---essentially the learner repeatedly queries the highest query rank concept. The approach is similar to that taken in \cite[Section 5]{angluin2017power} but with query rank in place of their notion of \emph{informative}. 

Now we informally describe the thicket max-min-algorithm. At stage $i$, the learner is given information of a concept class $\mc C_i .$ The learner picks the query 
$$A= \text{arg max} _ {A \in \mc C_i} \left( \text{min}_ {B \in \mc C_i}  \, d_{\mc C_i}( A,B) \right).$$ 
The algorithm halts if the learner has picked the actual concept $C$. If not, the teacher returns a random element $a_i \in \Delta (A,C)$ at which point the learner knows the value of $C (a_i).$ Then $$\mc C_{i+1} = (\mc C_i) _{a_i= C(a_i)}.$$ Let $T(\mc C)$ be the expected number of queries before the learner correctly identifies the target concept.

\begin{thm} The expected number of queries to learn a concept in a class $\mc C$ is less than or equal to $2 \Ldim (\mc C).$ 
\end{thm} 
\begin{proof} The expected drop in the Littlestone dimension of the concept class induced by any query before the algorithm terminates is at least ${1}/{2}$ by Theorem \ref{qrlowbound}; so the probability that the drop in the Littlestone dimension is positive is at least ${1}/{2}$ for any given query. So, from $2n$ queries, one expects at least $n$ drops in Littlestone dimension, at which point the class is learned. \end{proof} 



\section{Equivalence queries with random counterexamples and random targets} \label{learningmore} 

Let $\mc C$ consist the collection of intervals $\left\{\left(\frac{1}{n+1}, \frac{1}{n}\right) \, | \, n \in \m N \right\}$ with $\mu$ the Lebesgue measure on the unit interval. This concept class has Littlestone dimension one since any two concepts are disjoint. There is no upper bound on the number of expected queries (using the model with random counterexamples of the previous section) which is uniform over all targets.

To see why, suppose the learner guesses interval $\left(\frac{1}{n+1}, \frac{1}{n}\right)$ for some $n$. For any $\epsilon>0$ there is $N \in \m N$ such that with probability greater than $1-\epsilon$, the learner gets a counterexample from the interval they guessed, $\left(\frac{1}{n+1}, \frac{1}{n}\right)$. Of course, even with this additional information, no matter the learner's guess at any stage at which they have received only negative counterexamples, this is clearly still the case. Thus, there can be no bound on expected queries which is uniform over all target concepts. 

In this section we introduce an additional source of randomness which allows for learning over infinite classes $\mc C$.\footnote{One might also think of the random EQ learning of Angluin and Dohrn as analysing the maximum number of expected number of queries over all possible targets, while our model will analyze the \emph{expected} number of queries where the expectation is taken over the the concepts (with a fixed but arbitrary distribution) and over the counterexamples.} So, suppose $\mc C$ is a (possibly infinite) set of concepts on a set $X$. Suppose that we have probability measures $\mu $ on $X$ and $\tau$ on $\mc C$. Suppose a target $A \in \mc C$ is selected randomly according to the distribution $\tau$ and the counterexamples to equivalence queries are selected randomly according to the distribution $\mu.$

\begin{thm} Suppose that $\mc C$ is countable with finite Littlestone dimension $d$. There is an algorithm such that the expected number of queries over distributions $\mu$ on $X$ and $\tau$ on $\mc C$  is at most $\tilde{O}(d)$. 
\end{thm}
\begin{proof}
Let $\epsilon_k = \frac{1}{2^{k+1}}$ for $k \in \m N$. The idea of the algorithm is to run our earlier algorithm on a $1-\epsilon _k$ fraction of the concepts with respect to the measure $\tau$. 

At stage $k$ of the algorithm, we observe the following.
Since $\mc C$ is countable, enumerate the collection $\mc C = \{ C_i \}_{i \in \m N }.$ Then since $\sum_{i =1} ^ \infty P (C_i ) =1$, for any $\epsilon_k >0$, there is $N_k = N ( \epsilon_k ) \in \m N$ such that $\sum_{i =1} ^ \infty P (C_i ) \geq 1- \epsilon_k$.

Conditional on the target being among the first $N_k$ concepts, the next idea is to run the algorithm from the previous section on this finite set for $n$ steps where $n$ is such that the probability that we have not identified the target after $n$ steps is less than $\epsilon$, for some $0<\epsilon<1$. This number $n = n_{d, \epsilon}$ depends only on the Littlestone dimension and $\epsilon,$ but not on $N$ as we will explain. 

We now bound the probability that the algorithm has not terminated after $n$ steps, conditional on the target being in the first $N_k$ many concepts. Since at any step, the probability that the Littlestone dimension drops is at least $\frac{1}{2}$ by Theorem \ref{qrlowbound}, the probability that the algorithm has not terminated after $n$ steps is at most the probability of a binomial random variable with probability $\frac{1}{2}$ achieving at most $d-1$ successes in $n$ attempts, which is

$$
\sum_{k=0}^{d-1} \binom{n}{k} \left(\frac{1}{2}\right)^n 
 \le n^d / 2^n.
$$
Note that $n^d / 2^n < \epsilon $ whenever $n-d \log n  > \log \left ( \frac{1}{\epsilon} \right).$ Hence, $$n \ge \tilde{O}(d + \log(1/\epsilon))$$ is sufficient.

So at stage $k$, we run the algorithm for $n$ steps as specified above.  Either the target concept is found or we continue to stage $k+1$ on the larger concept class $N_k$.  
Since $$(1-\epsilon_1) \left(\sum_{k=1}^\infty \epsilon_k\right) = 1/2 \sum_{k=1}^\infty 1/2^{k+1} <1,$$ the expected total number of
queries is still bounded by $\tilde{O}(d + \log(1/\epsilon))$.\footnote{There isn't anything particularly special about the sequence $\epsilon_k$ that we chose. Any sequence $(\epsilon_k)$ going to zero whose sum converges can be seen to work in the algorithm, and affects only the constants in the expected number of steps, which we are not optimizing.}
\end{proof}




\section{Compression schemes and stability} \label{compression}

In this section, we follow the notation and definitions given in \cite{johnson2010compression} on \emph{compression schemes}, a notion due to Littlestone and Warmuth \cite{littlestone1986relating}. Roughly speaking, $\mc C$ admits a \emph{$d$-dimensional compression scheme} if, given any finite subset $F$ of $X$ and some $f\in \mc C$, there is a way of encoding the set $F$ with only $d$-many elements of $F$ in such a way that $F$ \emph{can be recovered}. 

We will give a formal definition, but we note that numerous variants of this idea appear throughout the literature, including as
size $d$-array compression \cite{ben1998combinatorial}. 
extended compression schemes with $b$ extra bits~\cite{floyd1995sample}, 
and as unlabeled compression schemes \cite{kuzmin2007unlabeled}.

The next definition, gives the notion of compression we will work with in this section; the notion is equivalent to the notion of a  $d$-compression with $b$ extra bits of \cite{floyd1995sample}. The equivalence of these two notions is proved by \cite[Proposition 2.1]{johnson2010compression}. In our compression schemes, the role of the $b$ extra bits is played by the reconstruction functions, and of course the number of extra bits can be bounded in terms of the number of reconstruction functions (and vice versa). Of course, one is interested in optimizing both the size of the compression and the number of reconstruction functions (extra bits) in general.

\begin{defn}
We say that a concept class $\mc C$ has a \emph{$d$-compression} if there is a compression function $\kappa : \mc C_{fin} \rightarrow X^{ d}$ and a finite set $\mc R$ of reconstruction functions $\rho : X^d  \rightarrow 2^X$ such that for any $f \in \mc C_{fin}$	

\begin{enumerate} 

\item $\kappa (f) \subseteq dom(f)$

\item 	$f = \rho (\kappa (f))|_{dom(f)} $ for at least one $\rho \in \mc R.$

\end{enumerate}

\end{defn}

We work with the above notion mainly because  it is the notion used in \cite{johnson2010compression}, and our goal is to improve a result of Laskowski and Johnson therein. That result was later improved by Laskowski and appears in the unpublished notes of \cite{guingonanip} (Theorem 4.1.3). When the original work on this result was completed, we were not aware of the work of \cite{guingonanip}, but as it turns out, our result improves both of these (the latter uses exponentially many reconstruction functions, while we use linearly many). 

\cite{johnson2010compression} 
prove that a concept class with finite Littlestone dimension has has an extended $d$-compression for some $d$.\footnote{Their result is formulated for the sets of realizations of first order formulas which are \emph{stable}, but their proofs work for general concept classes, and \cite{MLMT} explains that stable is equivalent to finite Littlestone dimension.} The precise value of $d$ is not determined there, but was conjectured to be the Littlestone dimension. In Theorem \ref{JLbound}, we will show that $d$ can be taken to be the Littlestone dimension and $d+1$ many reconstruction functions suffice.\footnote{After proving this, we became aware of the unpublished result of Laskowski appearing as \cite[Theorem 4.1.3]{guingonanip} which shows one can take $d$ to be the Littlestone dimension and uses $2^d$ many reconstruction functions.}

The question in  \cite{johnson2010compression}
is the analogue (for Littlestone dimension) of a well-known open question from VC-theory \citep{floyd1995sample}: is there a bound $A(d)$ linear in $d$ such that every class of VC-dimension $d$ has a compression scheme of size at most $A(d)$? In general there is known to be a bound that is at most exponential in $d$ \citep{moran2016sample}.

\begin{defn}
	Suppose $\Ldim(\mathcal{C}) = d$. Given a partial function $f$, say that $f$ is \emph{exceptional} for $\mathcal{C}$ if for all $a \in \dom(f)$,
	\[
		\mathcal{C}_{(a, f(a))} := \{ g \in \mathcal{C} \, | \, g(a) = f(a) \}
	\]
	has Littlestone dimension $d$.
\end{defn}

%

\begin{defn}
	Suppose $\Ldim(\mathcal{C}) = d$. Let $f_{\mathcal{C}}$ be the partial function given by
	\[
		f_{\mathcal{C}}(x) = \begin{cases}
			0 & \Ldim(\mathcal{C}_{(x,0)}) = d \\
			1 & \Ldim(\mathcal{C}_{(x,1)}) = d \\
			\mathrm{undefined} & \mathrm{otherwise.}
		\end{cases}
	\]
\end{defn}
It is clear that $f_{\mathcal{C}}$ extends any partial function exceptional for $\mathcal{C}$.

\begin{thm} \label{JLbound}
Any concept class $\mc C$ of Littlestone dimension $d$ has an extended $d$-compression with $(d+1)$-many reconstruction functions. 
\end{thm}
\begin{proof} 
If $d = 0$, then $\mathcal{C}$ is a singleton, and one reconstruction function suffices. So we may assume $d \geq 1$.	
	
Fix some $f \in \mc C_{fin}$ with domain $F$. 
We will run an algorithm to construct a tuple of length at most $d$ from $F$ by adding one element at each step of the algorithm. During each step of the algorithm, we also have a concept class $\mc C_i$, with $\mc C_0 = \mc C$ initially. 

If $f$ is exceptional in $\mc C_{i-1}$, then the algorithm halts. Otherwise, pick either: 
\begin{itemize}

\item  $a_i \in F$ such that $f(a_i)=1$ and 
\[
	(\mc C_{i-1} )_{(a_i,1)} := \{g \,| \, g \in \mc C_{i-1}, \, g(a_i)=1 \}
\] 
has Littlestone dimension less than $\Ldim(\mathcal{C}_{i-1})$. In this case, set $\mc C_i:= (\mc C_{i-1} )_{(a_i,1)}= \{g\,| \, g \in \mc C_{i-1}, \, g(a_i)=1 \}.$ 
 
\item $d_i \in F$ such that $f(d_i)=0$ and 
\[
	(\mc C _{i-1})_{(d_i,0)} := \{g \,| \, g \in \mc C_{i-1}, \, g(d_i)=0 \}
\] 
has Littlestone dimension less than $\Ldim(\mathcal{C}_{i-1})$. In this case, set $\mc C_i:= (\mc C_{i-1})_{(d_i,0)}.$ 

\end{itemize} 

We allow the algorithm to run for at most $d$ steps. There are two distinct cases. If our algorithm has run for $d$ steps, let $\kappa(f)$ be the tuple $(\bar a, \bar d)$ of all of the elements $a_i$ as above followed by all of the elements $d_i$ as above for $i=1, \ldots , d$. By choice of $a_i$ and $d_i$, this tuple consists of $d$ distinct elements. By construction the set 
\[
	\mc C_{(\bar a , \bar d)} := \{ g \in \mc C | \, g(a_i ) = 1, \, g(d_i) =0 \}
\] 
has Littlestone dimension $0$, that is, there is a unique concept in this class. So, given $(c_1, c_2, \ldots , c_n) \in X^d$ consisting of distinct elements, for $i=0, \ldots , d$,  we let $\rho_i(c_1, \ldots, c_n)$ be some $g$ belonging to
\[
	\{g \in \mc C \, | \, g(c_j)=1 \text{ for } j\leq i, \, g(c_j )=0 \text{ for } j>i \},
\]
if such a $g$ exists. By construction, for some $i$, the Littlestone dimension of the concept class $\{g \in \mc C \cap F \, | \, g(c_j)=1 \text{ for } j\leq i, \, g(c_j )=0  \text{ for } j>i \}$ is zero, and so $g$ is uniquely specified and will extend $f$. 

We handle cases where the algorithm halts early by augmenting two of the reconstruction functions $\rho_0$ and $\rho_1$ defined above. Because $\rho_0$ and $\rho_1$ have so far only been defined for tuples consisting of $d$ distinct elements, we can extend these to handle exceptional cases by generating tuples with duplicate elements. 

If the algorithm stops at some step $i>1$, then it has generated a tuple of length $i-1$ consisting of some elements $a_j$ and some elements $d_k$. Let $\bar a$ consist of the elements $a_j$ chosen during the algorithm, and let $\bar d$ consist of the elements $d_k$ chosen during the running of the algorithm. Observe that $f$ is exceptional for $\mathcal{C}_{(\bar{a}, \bar{d})}$. 

If $\bar{a}$ is not empty, with initial element $a'$, then let $\kappa(f) = (\bar a, a' , \bar d, a', \ldots , a') \in F^d$. From this tuple, one can recover $(\bar{a}, \bar{d})$ (assuming $\bar{a}$ is nonempty), so we let $\rho_1(\bar a, a' , \bar d, a', \ldots , a')$ be some total function extending $f_{\mathcal{C}_{(\bar{a}, \bar{d})}}$, which itself extends $f$. So $\rho_1(\bar{a}, \bar{d})$ extends $f$ whenever the algorithm halts before step $d$ is completed \emph{and} some $a_i$ was chosen at some point. If $\bar{a}$ is empty, then let $\kappa(f) = (\bar d , d', \ldots , d') \in F^d$, where $d'$ is the initial element of $\bar{d}$. From this tuple, one can recover $(\emptyset, \bar{d})$ (assuming $\bar{a}$ is empty), so we let $\rho_0(\bar d , d', \ldots , d')$ be total function extending $f_{\mathcal{C}_{(\emptyset, \bar{d})}}$, which itself extends $f$. Finally, if the algorithm terminates during step 1, then it has generated the empty tuple. In this case, let $\kappa(f) = (c, \ldots, c)$ for some $c \in F$. Then $\Ldim(\mathcal{C}) = \Ldim({\mathcal{C}}_{(c, l)})$ for some $l \in \{0,1\}$. In particular, if we have defined $\kappa(f') = (c, \ldots, c)$ above for some $f'$ where the algorithm only returns $c$ (rather than the empty tuple), then $1 - l = f'(c) \neq f(c)$, and so any such $f'$ is handled by $\rho_{1-l}$. So we may overwrite $\rho_l$ to set $\rho(c, \ldots, c)$ to be a total function extending $f_\mathcal{C}$, which itself extends $f$. For any tuple output by our algorithm, one of the reconstruction functions produces an extension of the original concept. 
\end{proof}

\section*{Acknowledgements}

This research was supported in part by award ECCS-2217023
from the National Science Foundation.

\newpage 
\bibliography{Research}

\begin{thebibliography}{12}
\providecommand{\natexlab}[1]{#1}
\providecommand{\url}[1]{\texttt{#1}}
\expandafter\ifx\csname urlstyle\endcsname\relax
  \providecommand{\doi}[1]{doi: #1}\else
  \providecommand{\doi}{doi: \begingroup \urlstyle{rm}\Url}\fi

\bibitem[Angluin and Dohrn(2017)]{angluin2017power}
Dana Angluin and Tyler Dohrn.
\newblock The power of random counterexamples.
\newblock In \emph{International Conference on Algorithmic Learning Theory},
  pages 452--465, 2017.

\bibitem[Ben-David and Litman(1998)]{ben1998combinatorial}
Shai Ben-David and Ami Litman.
\newblock Combinatorial variability of {Vapnik-Chervonenkis} classes with
  applications to sample compression schemes.
\newblock \emph{Discrete Applied Mathematics}, 86\penalty0 (1):\penalty0 3--25,
  1998.

\bibitem[Chase and Freitag(2019)]{MLMT}
Hunter Chase and James Freitag.
\newblock Model theory and machine learning.
\newblock \emph{Bulletin of Symbolic Logic}, 25\penalty0 (3):\penalty0
  319--332, 2019.

\bibitem[Floyd and Warmuth(1995)]{floyd1995sample}
Sally Floyd and Manfred Warmuth.
\newblock Sample compression, learnability, and the {Vapnik-Chervonenkis}
  dimension.
\newblock \emph{Machine learning}, 21\penalty0 (3):\penalty0 269--304, 1995.

\bibitem[Guingona()]{guingonanip}
Vincent Guingona.
\newblock {NIP} theories and computational learning theory.
\newblock \emph{\url{https://tigerweb.towson.edu/vguingona/NIPTCLT.pdf}}.

\bibitem[Johnson and Laskowski(2010)]{johnson2010compression}
Hunter~R Johnson and Michael~C Laskowski.
\newblock Compression schemes, stable definable families, and o-minimal
  structures.
\newblock \emph{Discrete \& Computational Geometry}, 43\penalty0 (4):\penalty0
  914--926, 2010.

\bibitem[Kumar et~al.(2021)Kumar, Chen, and Singla]{kumar2021teaching}
Akash Kumar, Yuxin Chen, and Adish Singla.
\newblock Teaching via best-case counterexamples in the
  learning-with-equivalence-queries paradigm.
\newblock \emph{Advances in Neural Information Processing Systems},
  34:\penalty0 26897--26910, 2021.

\bibitem[Kuzmin and Warmuth(2007)]{kuzmin2007unlabeled}
Dima Kuzmin and Manfred~K Warmuth.
\newblock Unlabeled compression schemes for maximum classes.
\newblock \emph{Journal of Machine Learning Research}, 8\penalty0 (9), 2007.

\bibitem[Littlestone(1988)]{littlestone1988learning}
Nick Littlestone.
\newblock Learning quickly when irrelevant attributes abound: A new
  linear-threshold algorithm.
\newblock \emph{Machine Learning}, 2\penalty0 (4):\penalty0 285--318, 1988.

\bibitem[Littlestone and Warmuth(1986)]{littlestone1986relating}
Nick Littlestone and Manfred Warmuth.
\newblock Relating data compression and learnability.
\newblock Technical report, University of California, Santa Cruz, 1986.

\bibitem[Moran and Yehudayoff(2016)]{moran2016sample}
Shay Moran and Amir Yehudayoff.
\newblock Sample compression schemes for {VC} classes.
\newblock \emph{Journal of the ACM (JACM)}, 63\penalty0 (3):\penalty0 21, 2016.

\bibitem[Warmuth(2003)]{10.1007/978-3-540-45167-9_60}
Manfred~K. Warmuth.
\newblock Compressing to vc dimension many points.
\newblock In Bernhard Sch{\"o}lkopf and Manfred~K. Warmuth, editors,
  \emph{Learning Theory and Kernel Machines}, pages 743--744, Berlin,
  Heidelberg, 2003. Springer Berlin Heidelberg.
\newblock ISBN 978-3-540-45167-9.

\end{thebibliography}

\end{document}